\documentclass{article}

\usepackage{times,natbib,fullpage}

\usepackage{microtype}
\usepackage{booktabs} 

\usepackage{hyperref}

\usepackage{algorithm}
\usepackage{algorithmic}

\usepackage{amsmath,amssymb,bm,bbold}
\usepackage{amsthm,framed,url}
\usepackage{color}
\usepackage{psfrag}
\usepackage{wrapfig}
\usepackage{graphicx,subcaption,tikz,epstopdf}
\usepackage{enumitem,array}
\usepackage{xfrac}
\usepackage{verbatim}
\usepackage{textcomp}
\usepackage{balance}
\usepackage{lipsum}
\usepackage{multirow,multicol}

\newtheorem{theorem}{Theorem}
\newtheorem{proposition}{Proposition}

\theoremstyle{definition}
\newtheorem{definition}{Definition}

\newcommand{\bOmega}{\boldsymbol{\varOmega}}
\newcommand{\bTheta}{\boldsymbol{\varTheta}}
\newcommand{\bM}{\boldsymbol{M}}

\newcommand{\bPi}{\boldsymbol{\varPi}}
\newcommand{\bH}{\boldsymbol{H}}
\newcommand{\bR}{\boldsymbol{R}}
\newcommand{\bS}{\boldsymbol{S}}
\newcommand{\bG}{\boldsymbol{G}}
\newcommand{\bXi}{\boldsymbol{\varXi}}
\newcommand{\bI}{\boldsymbol{I}}

\newcommand{\bQ}{\boldsymbol{Q}}

\newcommand{\bg}{\boldsymbol{g}}
\newcommand{\x}{\boldsymbol{x}}
\newcommand{\bp}{\boldsymbol{p}}
\newcommand{\bmu}{\boldsymbol{\mu}}

\newcommand{\bpsi}{\boldsymbol{\psi}}
\newcommand{\btheta}{\boldsymbol{\theta}}

\newcommand{\bbR}{\mathbf{R}}

\newcommand{\cC}{\mathcal{C}}
\newcommand{\cO}{\mathcal{O}}

\newcommand{\T}{{\!\top\!}}
\newcommand{\ones}{\boldsymbol{\mathit{1}}}

\DeclareMathOperator*{\minimize}{\textrm{minimize}}
\DeclareMathOperator{\E}{E}
\DeclareMathOperator{\cone}{cone}
\DeclareMathOperator{\rank}{rank}
\DeclareMathOperator{\krank}{krank}
\DeclareMathOperator{\spark}{spark}
\DeclareMathOperator{\tr}{Tr}
\DeclareMathOperator{\diag}{Diag}
\DeclareMathOperator{\vect}{vec}

\newcommand{\tone}[1]{\textcolor{blue}{#1}}
\newcommand{\ttwo}[1]{\textcolor{green}{#1}}

\newcommand{\tfour}[1]{\textcolor{purple}{#1}}
\newcommand{\tfive}[1]{\textcolor{orange}{#1}}

\newcommand{\tseven}[1]{\textcolor{olive}{#1}}
\newcommand{\teight}[1]{\textcolor{pink}{#1}}

\newcommand{\cpd}[1]{\text{\textlbrackdbl}#1\text{\textrbrackdbl}}

\title{\bf\Large Learning Hidden Markov Models from Pairwise Co-occurrences\\ with Application to Topic Modeling}
\author{\small
\begin{tabular}{ccc}
{\large Kejun Huang} & {\large Xiao Fu} & {\large Nicholas D. Sidiropoulos} \\
\hspace*{.03\textwidth} University of Minnesota \hspace*{.03\textwidth} & Oregon State University & University of Virginia \\
Minneapolis, MN 55455 & Corvallis, OR 97331 & Charlottesville, VA 22904 \\
\texttt{huang663@umn.edu} & \texttt{xiao.fu@oregonstate.edu} & \texttt{nikos@virginia.edu}
\end{tabular}
}
\date{}

\begin{document}

\maketitle

\begin{abstract}
We present a new algorithm for identifying the transition and emission probabilities of a hidden Markov model (HMM) from the emitted data. Expectation-maximization becomes computationally prohibitive for long observation records, which are often required for identification. The new algorithm is particularly suitable for cases where the available sample size is large enough to accurately estimate second-order output probabilities, but not higher-order ones. We show that if one is only able to obtain a reliable estimate of the pairwise co-occurrence probabilities of the emissions, it is still possible to uniquely identify the HMM if the emission probability is \emph{sufficiently scattered}. We apply our method to hidden topic Markov modeling, and demonstrate that we can learn topics with higher quality if documents are modeled as observations of HMMs sharing the same emission (topic) probability, compared to the simple but widely used bag-of-words model.
\end{abstract}

\section{Introduction}

Hidden Markov models (HMMs) are widely used in machine learning when the data samples are time \emph{dependent}, for example in speech recognition, language processing, and video analysis.
The graphical model of a HMM is shown in Figure~\ref{fig:hmm}. HMM models a (time-dependent) sequence of data $\{Y_t\}_{t=0}^T$ as indirect observations of an underlying Markov chain $\{X_t\}_{t=0}^T$ which is not available to us. Homogeneous HMMs are parsimonious models, in the sense that they are fully characterized by the transition probability $\Pr[X_{t+1}|X_{t}]$ and the emission probability $\Pr[Y_t|X_t]$  even though the size of the given data $\{Y_t\}_{t=0}^T$ can be very large. 

Consider a homogeneous HMM such that:
\begin{itemize}[noitemsep]
\item a latent variable $X_t$ can take $K$ possible outcomes $x_1,...,x_K$;
\item an ambient variable $Y_t$ can take $N$ possible outcomes $y_1,...,y_N$.
\end{itemize}
Recall that \cite{rabiner1986introduction,Ghahramani2001}:
\begin{itemize}[noitemsep]
\item Given both $\{X_t\}_{t=0}^T$ and $\{Y_t\}_{t=0}^T$, the complete joint probability factors, and we can easily estimate the transition probability $\Pr[X_{t+1}|X_{t}]$ and the emission probability $\Pr[Y_t|X_t]$.
\item Given only $\{Y_t\}_{t=0}^T$, but assuming we know the underlying transition and emission probabilities, we can calculate the observation likelihood using the forward algorithm, estimate the most likely hidden sequence using the Viterbi algorithm, and compute the posterior probability of the hidden states using the forward-backward algorithm.
\end{itemize}
The most natural problem setting, however, is when neither the hidden state sequence nor the underlying probabilities are known to us---we only have access to a sequence of observations, and our job is to reveal the HMM structure, characterized by the transition matrix $\Pr[X_{t+1}|X_{t}]$ and the emission probability $\Pr[Y_t|X_t]$ from the set of observations $\{Y_t\}_{t=0}^T$.

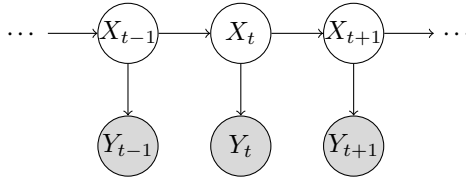
\begin{figure}[t]
\centering
\begin{tikzpicture}[]
\node at (   0,1.5) [circle,draw,minimum size=8mm,inner sep=-5pt] (Xt) {$X_t$};
\node at (-1.5,1.5) [circle,draw,minimum size=8mm,inner sep=-5pt] (Xt-1) {$X_{t-1}$};
\node at ( 1.5,1.5) [circle,draw,minimum size=8mm,inner sep=-5pt] (Xt+1) {$X_{t+1}$};
\node at (   0,0) [circle,draw,minimum size=8mm,inner sep=-5pt,fill=gray!30] (Yt) {$Y_t$};
\node at (-1.5,0) [circle,draw,minimum size=8mm,inner sep=-5pt,fill=gray!30] (Yt-1) {$Y_{t-1}$};
\node at ( 1.5,0) [circle,draw,minimum size=8mm,inner sep=-5pt,fill=gray!30] (Yt+1) {$Y_{t+1}$};
\node at (-2.9,1.5) (before) {\dots};
\node at ( 2.9,1.5) (after)  {\dots};
\draw [->] (before) -- (Xt-1);
\draw [->] (Xt+1) -- (after);
\draw [->] (Xt-1) -- (Xt);
\draw [->] (Xt) -- (Xt+1);
\draw [->] (Xt-1) -- (Yt-1);
\draw [->] (Xt) -- (Yt);
\draw [->] (Xt+1) -- (Yt+1);
\end{tikzpicture}
\caption{The graphical model of a HMM.}
\label{fig:hmm}
\end{figure}

\subsection{Related work}\label{sec:related}

The traditional way of learning a HMM from $\{Y_t\}_{t=0}^T$ is via expectation-maximization (EM) \cite{rabiner1986introduction}, in which the expectation step is performed by calling the forward-backward algorithm. This specific instance of EM is also called the Baum-Welch algorithm \cite{baum1970maximization,Ghahramani2001}. However, the complexity of Baum-Welch is prohibitive when $T$ is relatively large---the complexity of the forward-backward algorithm is $\cO(K^2T)$, but EM converges slowly, so the forward-backward algorithm must be called many times. This is a critical issue,  because a HMM can only be learned with high accuracy when the number of observation samples $T$ is large enough.

One way of designing scalable algorithms for learning HMMs is to work with sufficient statistics---a summary of the given observation sequence, whose size does not grow with~$T$. Throughout this paper we assume that the HMM process is stationary (time-invariant), which is true almost surely if the underlying Markov process is ergodic and the process has been going on for a reasonable amount of time. 
With $T$ large enough, we can accurately estimate the co-occurrence probability between two consecutive emissions $\Pr[Y_t,Y_{t+1}]$. According to the graphical model shown in Figure~\ref{fig:hmm}, it is easy to see that given the value of $X_t$, $Y_t$ is conditionally independent of all the other variables, leading to the factorization
\begin{align}\label{eq:hmm_fac2}
\Pr[Y_t,Y_{t+1}] 
= \sum_{k,j=1}^{K}&\Pr[Y_t|X_t=x_k]\Pr[Y_{t+1}|X_{t+1}=x_j]\Pr[X_t=x_k,X_{t+1}=x_j]
\end{align}
Let $\bOmega\in\bbR^{N\times N}$, $\bM\in\bbR^{N\times K}$, and $\bTheta\in\bbR^{K\times K}$, with their elements defined as
\begin{align*}
\varOmega_{n\ell} &= \Pr[Y_t=y_n,Y_{t+1}=y_\ell],\\
M_{nk} &= \Pr[Y_t=y_n|X_t=x_k],\\
\varTheta_{kj} &= \Pr[X_t=x_k,X_{t+1}=x_j].
\end{align*}
Then, equations \eqref{eq:hmm_fac2} can be written compactly as
\begin{align}\label{eq:MTM}
\bOmega = \bM\bTheta\bM^\T.
\end{align}
Noticing that $\eqref{eq:MTM}$ is a nonnegative matrix tri-factorization with a number of inconsequential constraints for $\bM$ and $\bTheta$ to properly represent probabilities, \citet{Vanluyten2008,Lakshminarayanan2010,Cybenko2011} proposed using nonnegative matrix factorization (NMF) to estimate the HMM probabilities. However, NMF-based methods have a serious shortcoming in this context: the tri-factorization~\eqref{eq:MTM} is in general not unique, because it is fairly easy to find a nonsingular matrix $\bQ$ such that both $\bM\bQ\geq0$ and $\bQ^{-1}\bTheta\bQ^{-\T}\geq0$, and then $\widetilde{\bM}=\bM\bQ$ and $\widetilde{\bTheta}=\bQ^{-1}\bTheta\bQ^{-\T}$ are equally good solutions in terms of reconstructing the co-occurrence matrix $\bOmega$. When we use $(\bM,\bTheta)$ and $(\widetilde{\bM},\widetilde{\bTheta})$ to perform HMM inference, such as estimating hidden states or predicting new emissions, the two models often yield completely different results, unless $\bQ$ is a permutation matrix.

A number of works propose to use \emph{tensor} methods to overcome the identifiability issue. Instead of working with the pairwise co-occurrence probabilities, they start by estimating the joint probabilities of three consecutive observations $\Pr[Y_{t-1},Y_t,Y_{t+1}]$. Noticing that these three random variables are conditionally independent given $X_t$, the triple-occurrence probability factors into
\begin{align*}
\Pr[Y_{t-1},Y_t,Y_{t+1}] = \sum_{k=1}^{K}\Pr[X_t=x_k]\Pr[Y_{t-1}|X_t=x_k]\Pr[Y_{t}|X_t=x_k]\Pr[Y_{t+1}|X_t=x_k],
\end{align*}
which admits a tensor \emph{canonical polyadic decomposition} (CPD) model \cite{Hsu2009,Anandkumar2012,Anandkumar2014}. Assuming $K\leq N$, the CPD is essentially unique if two of the three factor matrices have full column rank, and the other one is not rank one \cite{harshman1970foundations}; in the context of HMMs, this is equivalent to assuming $\bM$ and $\bTheta$ both have linearly independent columns, which is a relatively mild condition. The CPD is known to be unique under much more relaxed conditions \cite{sidiropoulos2017tensor}, but in order to uniquely retrieve the transition probability using the relationship
\[
\Pr[Y_{t+1}|X_{t}] = \sum_{j=1}^{K}
\Pr[Y_{t+1}|X_{t+1}\!=\!x_j]\Pr[X_{t+1}\!=\!x_j|X_t],
\]
$K\leq N$ is actually the best we can achieve using triple-occurrences without making further assumptions. 
\footnote{In the supplementary material, we prove that if the emission probability is \emph{generic} and the transition probability is \emph{sparse}, the HMM can be uniquely identified from triple-occurrence probability for \mbox{$K<N^2/16$} using the latest tensor identifiability result \cite{chiantini2012generic}.} 
A~salient feature in this case is that if the triple-occurrence probability $\Pr[Y_{t-1},Y_t,Y_{t+1}]$ is exactly given (meaning the rank of the triple-occurrence tensor is indeed smaller than $N$), the CPD can be efficiently calculated using generalized eigendecomposition and related algebraic methods \cite{sanchez1990tensorial,leurgans1993decomposition,DeLathauwer2004a}. These methods do not work well, however, when the low-rank tensor is perturbed; e.g., due to insufficient mixing / sample averaging of the triple occurrence probabilities.

It is also possible to handle cases where $K>N$. The key observation is that, given $X_t$, $Y_t$ is conditionally independent of $Y_{t-1},...,Y_{t-\tau}$ and $Y_{t+1},...,Y_{t+\tau}$. Then, grouping $Y_{t-1},...,Y_{t-\tau}$ into a single categorical variable taking $N^\tau$ possible outcomes, and $Y_{t+1},...,Y_{t+\tau}$ into another one, we can construct a much bigger tensor of size $N^\tau\times N^\tau\times N$, and then uniquely identify the underlying HMM structure with $K\gg N$ as long as certain linear independence requirements are satisfied for the conditional distribution of the \emph{grouped} variables \cite{Allman2009,bhaskara14a,Huang2016,Sharan2017}. It is intuitively clear that for fixed $N$, we need a much larger realization length $T$ in order to accurately estimate $(2\tau+1)$-occurrence probabilities as $\tau$ grows, which is the price we need to pay for learning a HMM with a larger number of hidden states.

\subsection{This paper}

The focus of this paper is on cases where $K \leq N$, and $T$ is large enough to obtain accurate estimate of $\Pr[Y_t,Y_{t+1}]$, but not large enough to accurately estimate triple or higher-order occurrence probabilities. We {\em prove} that it is actually possible to recover the latent structure of an HMM only from pairwise co-occurrence probabilities $\Pr[Y_t,Y_{t+1}]$, provided that the underlying emission probability $\Pr[Y_t|X_t]$ is \emph{sufficiently scattered}. Compared to the existing NMF-based HMM learning approaches, our formulation employs a different (determinant-based) criterion to ensure identifiability of the HMM parameters. Our matrix factorization approach resolves cases that cannot be handled by tensor methods, namely when $T$ is insufficient to estimate third-order probabilities, under an additional condition that is quite mild: that the emission probability matrix $\bM$ must be \emph{sufficiently scattered}, rather than simply full column-rank.

We apply our method to hidden topic Markov modeling (HTMM) \cite{gruber2007hidden}, in which case the number of hidden states (topics) is indeed much smaller than the number of ambient states (words). HTMM goes beyond the simple and widely used bag-of-words model by assuming that (ordered) words in a document are emitted from a hidden topic sequence that evolves according to a Markov model.  We show improved performance on real data when using this simple and intuitive model to take word ordering into account when learning topics, which also benefits from our identifiability guaranteed matrix factorization method. 

As an illustrative example, we showcase the inferred topic of each word in a news article (removing stop words) in Figure~\ref{fig:topic}, taken from the Reuters21578 data set obtained at \cite{reuters21578}.
As we can see, HTMM gets much more consistent and smooth inferred topics compared to that obtained from a bag-of-words model (cf. supplementary material for details). This result agrees with human understanding.

\begin{figure}[t]
\begin{framed}
\teight{china} \tfive{daily} \tfour{vermin eat} \tseven{pct} \tone{grain stocks survey provinces} \tseven{and} \tone{cities showed} \tfour{vermin consume} \tseven{and pct} \teight{china} \tone{grain stocks} \teight{china} \tfive{daily} \tseven{each} \tone{year} \tseven{mln} \teight{tonnes} \tseven{pct} \teight{china} \tfive{fruit} \ttwo{output} \tone{left} \tseven{rot and mln} \teight{tonnes} \tseven{pct} \tfour{vegetables} \tfive{paper} \teight{blamed} \tseven{waste inadequate} \tfive{storage} \tseven{and} \tone{bad} \tfour{preservation} \tseven{methods} \tone{government} \tseven{had launched} \tfive{national} \tone{programme reduce} \tseven{waste calling for} \tone{improved} \tfive{technology storage} \tseven{and} \tfour{preservation} \tseven{and} \tone{greater production} \tfour{additives} \tfive{paper} \tseven{gave details}

\bigskip

\tfour{china} \ttwo{daily} \tfour{vermin eat} \teight{pct} \tfour{grain} \ttwo{stocks} \tfour{survey provinces and cities showed vermin consume and pct} \teight{china} \tfour{grain stocks} \ttwo{china} \tfour{daily} \tfour{each} \tseven{year} \tfour{mln} \teight{tonnes} \tfour{pct} \teight{china} \tfour{fruit output} \ttwo{left} \tfour{rot and mln} \teight{tonnes} \tfour{pct} \teight{vegetables} \tfour{paper blamed waste} \tseven{inadequate} \tfour{storage} \ttwo{and} \tfour{bad preservation methods government had launched national programme} \teight{reduce} \tfour{waste} \tseven{calling} \tfour{for improved} \teight{technology} \tfour{storage} \ttwo{and} \tfour{preservation and greater production} \ttwo{additives} \tfour{paper gave details}
\end{framed}
\caption{Inferred topics of the words shown in different colors, obtained by probabilistic latent semantic analysis (top) and hidden topic Markov model (bottom).}
\label{fig:topic}
\end{figure}

\section{Second-order vs. Third-order Learning}

We start by arguing that for the same observation data $\{Y_t\}_{t=0}^T$, the estimate of the pairwise co-occurrence probability $\Pr[Y_t,Y_{t+1}]$ is always more accurate than that of the triple co-occurrence probability $\Pr[Y_{t-1},Y_t,Y_{t+1}]$. 

Let us first explicitly describe the estimator we use for these probabilities. For each observation $Y_t$, we define a coordinate vector $\bpsi_t\in\bbR^K$, and $\bpsi_t=\bm{e}_k$ if $Y_t=y_k$. The natural estimator for the pairwise co-occurrence probability matrix $\bOmega$ is
\begin{equation}\label{eq:omega2}
\widehat{\bOmega} = \frac{1}{T}\sum_{t=0}^{T-1}\bpsi_t\bpsi_{t+1}^\T,
\end{equation}
and similarly for the triple co-occurrence probability $\underline{\bOmega_3}$
\begin{equation}\label{eq:omega3}
\widehat{\underline{\bOmega_3}} = \frac{1}{T-1}\sum_{t=1}^{T-1}\bpsi_{t-1}\circ\bpsi_t\circ\bpsi_{t+1},
\end{equation}
where $\circ$ denotes vector outer-product. 
\footnote{In some literature $\circ$ is written as the Kronecker product $\otimes$. Strictly speaking, the Kronecker product of three vectors is a very long vector, not a three-way array. For this reason, we chose to use $\circ$ instead of $\otimes$.}

The first observation is that both $\widehat{\bOmega}$ and $\widehat{\underline{\bOmega_3}}$ are unbiased estimators: Obviously $\E(\bpsi_t\bpsi_{t+1}^\T)=\bOmega$ and likewise for the triple-occurrences, and taking their averages does not change the expectation. However, the individual terms in the summation are not independent of each other, making it hard to determine how fast estimates converge to their expectation. The state-of-the-art concentration result for HMMs \cite{kontorovich2006measure} states that for any 1-Lipschitz function $f$
\[
\Pr[|f(\{Y_t\})-\E f(\{Y_t\})|>\epsilon]
\leq 2\exp\left(-T\epsilon^2/c\right),
\]
where $c$ is a constant that only depends on the specific HMM structure but not on the function $f$ (cf. \cite{kontorovich2006measure} for details). Taking $f$ as any entry in $\widehat{\bOmega}$ or $\widehat{\underline{\bOmega_3}}$, we can check that indeed it is 1-Lipschitz, meaning as $T$ goes to infinity, both estimators converge to their expectation with negligible fluctuations.

We now prove that for a given set of observations $\{Y_t\}_{t=0}^T$, $\widehat{\bOmega}$ is always going to be more accurate than $\widehat{\underline{\bOmega_3}}$. Since both of them represent probabilities, we use two common metrics to measure the differences between the estimators and their expectations, the Kullback-Leibler divergence $D_\text{KL}(\cdot)$ and the total-variation difference $D_\text{TV}(\cdot)$.

\begin{proposition}\label{ppst:2>3}
Let $\widehat{\bOmega}$ and $\widehat{\underline{\bOmega_3}}$ be obtained from the same set of observations $\{Y_t\}_{t=0}^T$, we have that
\begin{align*}
D_\textup{KL}(\widehat{\bOmega}\|\bOmega) &\leq 
D_\textup{KL}(\widehat{\underline{\bOmega_3}}\|\underline{\bOmega_3})
\qquad\text{and}\\
D_\textup{TV}(\widehat{\bOmega}\|\bOmega) &\leq 
D_\textup{TV}(\widehat{\underline{\bOmega_3}}\|\underline{\bOmega_3}).
\end{align*}
\end{proposition}
The proof of Proposition~\ref{ppst:2>3} is relegated to the supplementary material.

\section{Identifiability of HMMs from Pairwise  Co-occurrence Probabilities}

The arguments made in the previous section motivate going back to matrix factorization methods for learning a HMM when the realization length $T$ is not large enough to obtain accurate estimates of triple co-occurrence probabilities. As we have explained in \S\ref{sec:related}, the co-occurrence probability matrix $\bOmega$ admits a nonnegative matrix tri-factorization model \eqref{eq:MTM}. There are a number of additional equality constraints. Columns of $\bM$ represent conditional distributions, so $\ones^\T\bM=\ones^\T$. Matrix $\bTheta$ represents the joint distribution between two consecutive Markovian variables, therefore $\ones^\T\bTheta\ones=1$. Furthermore, we have that $\bTheta\ones$ and $\bTheta^\T\ones$ represent $\Pr[X_t]$ and $\Pr[X_{t+1}]$ respectively, and since we assume that the Markov chain is stationary, they are the same, i.e., $\bTheta\ones=\bTheta^\T\ones$. Notice that this does not imply that $\bTheta$ is symmetric, and in fact it is often not symmetric.

%
%

\citet{huang2016nips} considered a factorization model similar to \eqref{eq:MTM} in a different context, and showed that identifiability can be achieved under a reasonable assumption called \emph{sufficiently scattered}, defined as follows.

\begin{definition}[sufficiently scattered]\label{def:suf_scat}
Let $\cone(\bM^\T)^*$ denote the polyhedral cone $\{\x:\bM\x\geq 0\}$, and $\cC$ denote the elliptical cone $\{\x:\|\x\|\leq\ones^\T\x\}$. Matrix $\bM$ is called \textbf{sufficiently scattered} if it satisfies that:
(i)  $\cone(\bM^\T)^*\subseteq\cC$, and 
(ii) $\cone({\bM}^\T)^\ast\cap{\rm bd}\cC=\{\lambda {\bm e}_k:\lambda\geq 0,k=1,...,K\}$,
where ${\rm bd}\cC$ denotes the boundary of $\cC$, $\{\x:\|\x\|=\ones^{\T}\x\}$.
\end{definition}

The sufficiently scattered condition was first proposed in \cite{huang2014tsp} to establish uniqueness conditions for the widely used \emph{nonnegative matrix factorization} (NMF). For the NMF model $\bOmega=\bm{WH}^\T$, if both $\bm{W}$ and $\bm{H}$ are sufficiently scattered, then the nonnegative decomposition is unique up to column permutation and scaling. Follow up work strengthened and extended the identifiability results based on this geometry inspired condition \cite{fu2015bss,huang2016nips,fu2017spl}. A similar tri-factorization model was considered in \cite{huang2016nips} in the context of bag-of-words topic modeling, and it was shown that among all feasible solutions of~\eqref{eq:MTM}, if we find one that minimizes $|\det\bTheta|$, then it recovers the ground-truth latent factors $\bM$ and $\bTheta$, assuming the ground-truth $\bM$ is sufficiently scattered. In our present context, we therefore propose the following problem formulation:
\begin{subequations}\label{prob:main}
\begin{align}
\minimize_{\bTheta,\bM}~~~ & |\det\bTheta| \\
\textrm{subject to}~~~ & \bOmega=\bM\bTheta\bM^\T, \label{eq:Omega}\\
& \bTheta\geq0, \bTheta\ones=\bTheta^\T\ones, \ones^\T\bTheta\ones=1, \label{eq:Theta}\\
& \bM\geq0, \ones^\T\bM=\ones^\T. 
\end{align}
\end{subequations}

Regarding Problem~\eqref{prob:main}, we have the following identifiability result.
\begin{theorem}\label{thm:unique}
\cite{huang2016nips}
Suppose $\bOmega$ is constructed as 
$\bOmega=\bM_\natural\bTheta_\natural\bM_\natural^\T$,
where $\bM_\natural$ and $\bTheta_\natural$ satisfy the constraints in~\eqref{prob:main}, and in addition
(i) $\rank(\bTheta_\natural)=K$ and
(ii) $\bM_\natural$ is \textup{\bf sufficiently scattered}.
Let $(\bM_\star,\bTheta_\star)$ be an optimal solution for~\eqref{prob:main}, then there must exist a permutation matrix $\bPi\in\bbR^{K\times K}$ such that
\[
\bM_\natural = \bM_\star\bPi, \qquad
\bTheta_\natural = \bPi^\T\bTheta_\star\bPi.
\]
\end{theorem}
One may notice that in \cite{huang2016nips}, there are no constraints on the core matrix $\bTheta$ as we do in~\eqref{eq:Theta}. In terms of identifiability, it is easy to see that if the ground-truth $\bTheta_\natural$ satisfies~\eqref{eq:Theta}, solving~\eqref{prob:main} even without~\eqref{eq:Theta} will produce a solution $\bTheta_\star$ that satisfies~\eqref{eq:Theta}, thanks to uniqueness. In practice when we are given a less accurate $\bOmega$, such ``redundant'' information will help us improve the estimation error, but that goes beyond identifiability consederations.

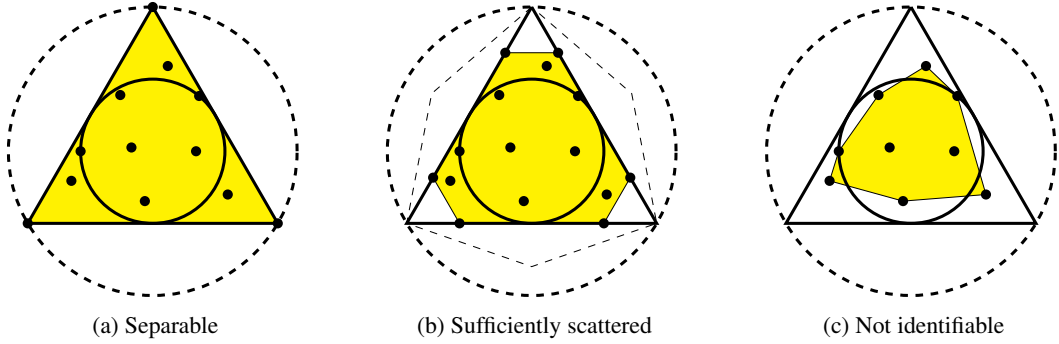
\begin{figure}[t]
\centering
\begin{subfigure}[t]{.3\linewidth}
	\centering
	\resizebox{.8\linewidth}{!}{\begin{tikzpicture}
\fill[fill=yellow] (90:2) -- (210:2) -- (-30:2);
\fill ( 90:2) circle(2pt);
\fill (210:2) circle(2pt);
\fill (-30:2) circle(2pt);

\fill ( 50:1) circle(2pt);
\fill ( 80:1.2) circle(2pt);
\fill (120:.9) circle(2pt);
\fill (180:1) circle(2pt);
\fill (200:1.2) circle(2pt);
\fill (-99:.7) circle(2pt);
\fill (-30:1.2) circle(2pt);

\fill (  0:.6) circle(2pt);
\fill (170:.3) circle(2pt);

\draw[very thick] (90:2) -- (210:2) -- (-30:2) -- (90:2);
\draw[very thick] (0,0) circle (1cm);
\draw[very thick, dashed] (0,0) circle (2cm);
\end{tikzpicture}}
	\caption{Separable}\label{fig:geo-separable}
\end{subfigure}
\begin{subfigure}[t]{.3\linewidth}
	\centering
	\resizebox{.8\linewidth}{!}{\begin{tikzpicture}
\filldraw[fill=yellow]
( 75:1.4142) -- (105:1.4142) -- (195:1.4142) -- (225:1.4142) -- 
(-45:1.4142) -- (-15:1.4142) -- ( 75:1.4142);

\fill ( 75:1.4142) circle(2pt);
\fill (105:1.4142) circle(2pt);
\fill (195:1.4142) circle(2pt);
\fill (225:1.4142) circle(2pt);
\fill (-45:1.4142) circle(2pt);
\fill (-15:1.4142) circle(2pt);

\draw[dashed] (90:2) -- (150:1.6) -- (210:2) -- (-90:1.6) -- (-30:2) -- (30:1.6) -- (90:2);

\fill ( 50:1) circle(2pt);
\fill ( 80:1.2) circle(2pt);
\fill (120:.9) circle(2pt);
\fill (180:1) circle(2pt);
\fill (200:1.2) circle(2pt);
\fill (-99:.7) circle(2pt);
\fill (-30:1.2) circle(2pt);

\fill (  0:.6) circle(2pt);
\fill (170:.3) circle(2pt);

\draw[very thick] (90:2cm) -- (210:2cm) -- (-30:2cm) -- (90:2cm);
\draw[very thick] (0,0) circle (1cm);
\draw[very thick, dashed] (0,0) circle (2cm);
\end{tikzpicture}}
	\caption{Sufficiently scattered}\label{fig:geo-scattered}
\end{subfigure}
\begin{subfigure}[t]{.3\linewidth}
	\centering
	\resizebox{.8\linewidth}{!}{\begin{tikzpicture}
\filldraw[fill=yellow] 
(50:1) -- (80:1.2) -- (120:.9) -- (180:1) -- (200:1.2) -- (-99:.7) -- (-30:1.2) -- (50:1);

\fill ( 50:1) circle(2pt);
\fill ( 80:1.2) circle(2pt);
\fill (120:.9) circle(2pt);
\fill (180:1) circle(2pt);
\fill (200:1.2) circle(2pt);
\fill (-99:.7) circle(2pt);
\fill (-30:1.2) circle(2pt);

\fill (  0:.6) circle(2pt);
\fill (170:.3) circle(2pt);

\draw[very thick] (90:2cm) -- (210:2cm) -- (-30:2cm) -- (90:2cm);
\draw[very thick] (0,0) circle (1cm);
\draw[very thick, dashed] (0,0) circle (2cm);
\end{tikzpicture}}
	\caption{Not identifiable}\label{fig:geo-no_id}
\end{subfigure}
\caption{A geometric illustration of the sufficiently scattered condition (middle), a special case that is separable (left), and a case that is not identifiable (right).}
\label{fig:geo}
\end{figure}

The proof of Theorem~\ref{thm:unique} is referred to \cite{huang2016nips}.
Here we provide some insights on this geometry-inspired sufficiently scattered condition, and discuss why it is a reasonable (and thus practical) assumption. The notation $\cone(\bM^\T)^*=\{\x:\bM\x\geq0\}$ comes from the convention in convex analysis that it is the \emph{dual cone} of the conical hull of the row vectors of $\bM$, i.e., $\cone(\bM^\T)=\{\bM^\T\bm{\alpha}:\bm{\alpha}\geq0\}$. Similarly, we can derive that the dual cone of $\cC$ is $\cC^*=\{\x:\|\x\|\leq\ones^\T\x/\sqrt{K-1}\}$. A useful property of the dual cone is that for two convex cones $\mathcal{A}$ and $\mathcal{B}$, $\mathcal{A}\subseteq\mathcal{B}$ iff $\mathcal{B}^*\subseteq\mathcal{A}^*$. Therefore, the first requirement of sufficiently scattered in Definition~\ref{def:suf_scat} equivalently means
\[
\cC^*\subseteq\cone(\bM^\T).
\]
We give a geometric illustration of the sufficiently scattered condition in Figure~\ref{fig:geo-scattered} for $K=3$, and we focus on the 2-dimensional plane $\ones^\T\x=1$. The intersection between this plane and the nonnegative orthant is the probability simplex, which is the triangle in Figure~\ref{fig:geo-scattered}. The outer circle represents $\cC$, and the inner circle represents $\cC^*$, again intersecting with the plane, respectively. The rows of $\bM$ are scaled to sum up to one, and they are represented by black dots in Figure~\ref{fig:geo-scattered}. Their conical hull is represented by the shaded region. The polygon with dashed lines represents the dual of $\cone(\bM^\T)$, which is indeed a subset of $\cC$, and touches the boundary of $\cC$ only at the coordinate vectors.

Figure~\ref{fig:geo-separable} shows a special case of sufficiently scattered called \emph{separability}, which first appeared in \cite{donoho2004does} also to establish uniqueness of NMF. In this case, all the coordinate vectors appear in rows of $\bM$, therefore $\cone(\bM)$ equals the nonnegative orthant. It makes sense that this condition makes the identification problem easier, but it is also a very restrictive assumption. The sufficiently scattered condition, on the other hand, only requires that the shaded region contains the inner circle, as shown in Figure~\ref{fig:geo-scattered}. Intuitively this requires that the rows of $\bM$ be ``well scattered'' in the probability simplex, but not to the extent of ``separable''. Separability-based HMM identification has been considered in \cite{Barlier2015,Glaude2015}. However, the way they construct second-order statistics is very different from ours. Figure~\ref{fig:geo-no_id} shows a case where $\bM$ is not sufficiently scattered, and it also happens to be a case where $\bM$ is not identifiable.

As we can see, the elliptical cone $\cC^*$ is tangent to all the facets of the nonnegative orthant. As a result, for $\bM$ to be sufficiently scattered, it is necessary that there are enough rows of $\bM$ lie on the boundary of the nonnegative orthant, i.e., $\bM$ is relatively sparse. Specifically, if $\bM$ is sufficiently scattered, then each column of $\bM$ contains at least $K-1$ zeros \cite{huang2014tsp}. This is a very important insight, as exactly checking whether a matrix is sufficiently scattered may be computationally hard. In the present paper we further show the following result.
\begin{proposition}\label{ppst:volume}
The ratio between the volume of the hyperball obtained by intersecting $\ones^\T\x=1$ and $\cC^*$ and the probability simplex is
\begin{equation}\label{eq:ratio}
\frac{1}{\sqrt{\pi K}}\left(\frac{4\pi}{K(K-1)}\right)^{\frac{K-1}{2}}
\Gamma\left(\frac{K}{2}\right).
\end{equation}
\end{proposition}
The proof is given in the supplementary material. 
As $K$ grows larger, the volume ratio \eqref{eq:ratio} goes to zero at a super-exponential decay rate. This implies that the volume of the inner sphere quickly becomes negligible compared to the volume of the probability simplex, as $K$ becomes moderately large. The take home point is that, for a practical choice of $K$, say $K\geq10$, as long as $\bM$ satisfies that each column contains at least $K$ zeros, and the positions of the zeros appear relatively random, it is very likely that it is sufficiently scattered, and thus can be uniquely recovered via solving~\eqref{prob:main}.

\section{Algorithm}
Our identifiability analysis based on the sufficiently scattered condition poses an interesting non-convex optimization problem~\eqref{prob:main}. First of all, the given co-occurrence probability $\bOmega$ may not be exact, therefore it may not be a good idea to put~\eqref{eq:Omega} as a hard constraint. For algorithm design, we propose the following modification to problem~\eqref{prob:main}.
\begin{align}\label{prob:alg}
\minimize_{\bTheta,\bM}~~~ & 
\sum_{n,\ell=1}^{N}-\varOmega_{n\ell}\log\!\!\sum_{k,j=1}^{K}\!\!M_{nk}\varTheta_{kj}M_{\ell j} 
 + \lambda|\det\bTheta| \nonumber\\
\textrm{subject to}~~~ & \bM\geq0, \ones^\T\bM=\ones^\T, \\
& \bTheta\geq0, \bTheta\ones=\bTheta^\T\ones, \ones^\T\bTheta\ones=1. \nonumber
\end{align}
In the loss function of~\eqref{prob:alg}, the first term is the Kullback-Leibler distance between the empirical probability $\bOmega$ and the parameterized version $\bM\bTheta\bM^\T$ (ignoring a constant), and the second term is our identifiability-driven regularization. We need to tune the parameter $\lambda$ to yield good estimation results. However, intuitively we should use a $\lambda$ with a relatively small value. Suppose $\bOmega$ is sufficiently accurate, then the priority is to minimize the difference between $\bOmega$ and $\bM\bTheta\bM^\T$; when there exist equally good fits, then the second term comes into play and helps us pick out a solution that is \emph{sufficiently scattered}.

Noticing that the constraints of~\eqref{prob:alg} are all convex, but not the loss function, we propose to design an iterative algorithm to solve~\eqref{prob:alg} using successive convex approximation. At iteration $r$ when the updates are $\bTheta^r$ and $\bM^r$, we define
\begin{align}\label{eq:posterior}
\varPi_{n\ell kj}^r = M_{nk}^r\varTheta_{kj}^rM_{\ell j}^r \bigg/ 
\sum_{\kappa,\iota=1}^{K}M_{n\kappa}^r\varTheta_{\kappa\iota}^rM_{\ell\iota}^r.
\end{align}
Obviously, $\varPi_{n\ell kj}^r\geq0$ and $\sum_{k,j=1}^{K}\varPi_{n\ell kj}^r=1$, which defines a probability distribution for fixed $n$ and $\ell$.
Using Jensen's inequality~\cite{jensen1906fonctions}, we have that
\begin{align}\label{eq:upper1}
-\varOmega_{n\ell}\log\sum_{k,j=1}^{K}M_{nk}\varTheta_{kj}M_{\ell j} \leq \sum_{k,j=1}^{K}-\varOmega_{n\ell}\varPi_{n\ell kj}^r
\left(\log M_{nk} + \log\varTheta_{kj} + \log M_{\ell j} - \log\varPi_{n\ell kj}^r\right)
\end{align}
which defines a convex and locally tight upperbound for the first term in the loss function of \eqref{prob:alg}. Regarding the second term in the loss of~\eqref{prob:alg}, we propose to simply take the linear approximation
\begin{align}\label{eq:approx2}
|\!\det\!\bTheta| \approx |\!\det\!\bTheta^r| + 
|\!\det\!\bTheta^r|\tr\!\left( (\bTheta^r)^{\!-\!1\!}(\bTheta\!-\!\bTheta^r) \right) 
\end{align}

Combining~\eqref{eq:upper1} and~\eqref{eq:approx2}, our successive convex approximation algorithm tries to solve the following convex problem at iteration $r$:
\begin{align}\label{prob:iter}
\minimize_{\bTheta,\bM}~~ & \sum_{n,\ell=1}^{N}\sum_{k,j=1}^{K}
-\varOmega_{n\ell}\varPi_{n\ell kj}^r \left(\log M_{nk} + \log M_{\ell j} + \log\varTheta_{kj} \right) 
+ \lambda\sum_{k,j=1}^{K}\varXi_{kj}^r\varTheta_{kj} \\
\textrm{subject to}~~ & \bM\geq0, \ones^\T\bM=\ones^\T, \nonumber \\
& \bTheta\geq0, \bTheta\ones=\bTheta^\T\ones, \ones^\T\bTheta\ones=1,  \nonumber
\end{align}
where we define $\bm{\varXi}^r=|\det\bTheta^r|(\bTheta^r)^{-\T}$.
Problem~\eqref{prob:iter} decouples with respect to $\bM$ and $\bTheta$, so we can work out their updates individually.

The update of $\bM$ admits a simple closed form solution, which can be derived via checking the KKT conditions. We denote the dual variable corresponding to $\ones^\T\bM=\ones^\T$ as $\bmu\in\bbR^K$. Setting the gradient of the Lagrangian with respect to $M_{nk}$ equal to zero, we have
\[
M_{nk} = 
\sum_{\ell=1}^N\sum_{j=1}^{K}\left(\varOmega_{n\ell}\varPi_{n\ell kj}^r+\varOmega_{\ell n}\varPi_{\ell njk}^r\right) \bigg/ \mu_k
\]
and $\bmu$ should be chosen so that the constraint $\ones^\T\bM=\ones^\T$ is satisfied, which amounts to a simple re-scaling.

The update of $\bTheta$ is not as simple as a closed form expression, but it can still be obtained very efficiently. Noticing that the nonnegativity constraint is implicitly implied by the individual $\log$ functions in the loss function, we propose to solve it using Newton's method with equality constraints \citep[\S10.2]{boyd2004convex}. Although Newton's method requires solving a linear system of equations with $K^2$ number of variables in each iteration, there is special structure we can exploit to reduce the per-iteration complexity down to $\mathcal{O}(K^3)$: The Hessian of the loss function of~\eqref{prob:iter} is diagonal, and the linear equality constraints are highly structured; using block elimination \citep[\S10.4.2]{boyd2004convex}, we ultimately only need to solve a positive definite linear system with $K$ variables. Together with the quadratic convergence rate of Newton's method, the complexity of updating $\bTheta$ is $\mathcal{O}(K^3\log\log\frac{1}{\varepsilon})$, where $\varepsilon$ is the desired accuracy for the $\bTheta$ update. 
Noticing that the complexity of a naive implementation of Newton's method would be $\mathcal{O}(K^6\log\log\frac{1}{\varepsilon})$, the difference is big for moderately large $K$.
The in-line implementation of this tailored Newton's method \textsc{ThetaUpdate} and the detailed derivation can be found in the supplementary material.

\begin{algorithm}[t]
\caption{Proposed Algorithm}\label{alg:hmm_id}
\begin{algorithmic}[1]
\REQUIRE $\lambda>0$
\STATE initialize $\bM$ using \cite{huang2016nips}
\STATE initialize $\bTheta\leftarrow \frac{1}{K(K+1)}(\bI+\ones\!\ones^\T)$
\REPEAT
\STATE $\widetilde{\bOmega}\leftarrow\bOmega\big/\bM\bTheta\bM^{\T}$
\COMMENT{element-wise division}
\STATE $\widetilde{\bM}\leftarrow \bM \ast \left( \widetilde{\bOmega}\bM\bTheta^{\T} + \widetilde{\bOmega}^\T\bM\bTheta \right)$
\STATE $\widetilde{\bTheta}\leftarrow\bM^{\T}\widetilde{\bOmega}\bM$
\STATE $\widetilde{\bM}\leftarrow \widetilde{\bM} \diag(\ones^\T\widetilde{\bM})^{-1}$
\STATE $\widetilde{\bTheta}\leftarrow$ \textsc{ThetaUpdate} \COMMENT{cf. supplementary}
\STATE $(\bM,\bTheta)\leftarrow$ Amijo line search between $(\bM,\bTheta)$ and $(\widetilde{\bM},\widetilde{\bTheta})$
\UNTIL{convergence}
\RETURN $\bM$ and $\bTheta$
\end{algorithmic}
\end{algorithm}

The entire proposed algorithm to solve Problem~\eqref{prob:alg} is summarized in Algorithm~\ref{alg:hmm_id}. Notice that there is an additional line-search step to ensure decrease of the loss function. The constraint set of \eqref{prob:alg} is convex, so the line-search step will not incur infeasibility. Computationally, we find that any operation that involves $\varPi_{n\ell kj}^r$ can be carried out succinctly by defining the intermediate matrix $\widetilde{\bOmega}=\bOmega/\bM\bTheta\bM^{\T}$, where ``$/$'' denotes element-wise division between two matrices of the same size.
The per-iteration complexity of Algorithm~\ref{alg:hmm_id} is completely dominated by the operations that involve computing with $\widetilde{\bOmega}$, notably comparing with that of \textsc{Theta-Update}. In terms of initialization, which is important since we are optimizing a non-convex problem, we propose to use the method by \citet{huang2016nips} to obtain an initialization for $\bM$; for $\bTheta$, it is best if we start with a feasible point (so that the Newton's iterates will remain feasible), and a simple choice is scaling the matrix $\bI+\ones\!\ones^\T$ to sum up to one. Finally, we show that this algorithm converges to a stationary point of Problem~\eqref{prob:alg}, with proof relegated to the supplementary material based on \cite{razaviyayn2013unified}.

\begin{proposition}
Assume \textsc{ThetaUpdate} solves Problem~\eqref{prob:iter} with respect to $\bTheta$ exactly, then
Algorithm~\ref{alg:hmm_id} converges to a stationary point of Problem~\eqref{prob:alg}.
\end{proposition}

\begin{figure}[t]
\centering
\includegraphics[width=\linewidth]{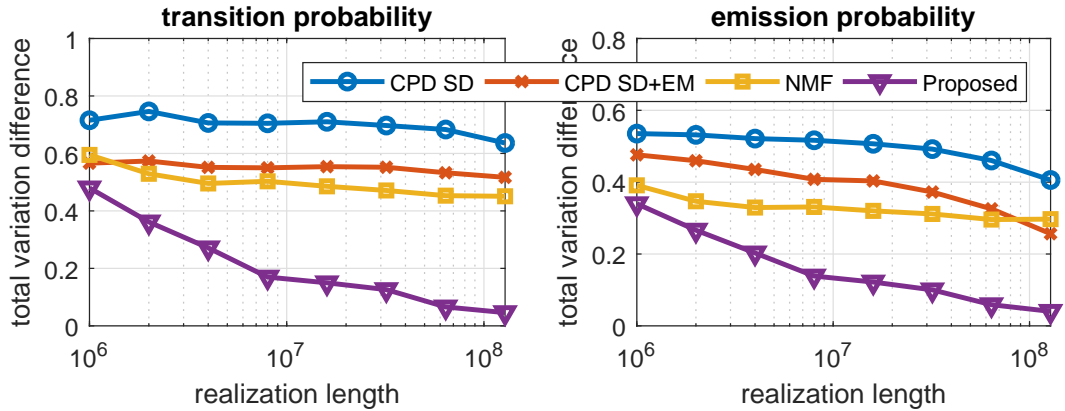}
\caption{The total variation difference between the ground truth and estimated transition probability (top) and emission probability (bottom). The total variation difference of the emission probabilities is calculated as $\frac{1}{2K}\|\bM_\natural-\bM_\star\|_1$, since each column of the matrices indicates a (conditional) probability, and the total variation difference is equal to one half of the $L_1$-norm; and similarly for that of the transition probabilities after rescaling the rows of $\bOmega_\natural$ and $\bOmega_\star$ to sum up to one. The result is averaged over 10 random problem instances.}
\label{fig:sim_exp1}
\end{figure}


\section{Validation on Synthetic Data}

In this section we validate the identifiability performance on synthetic data. In this case, the underlying transition and emission probabilities are generated synthetically, and we compare them with the estimated ones to evaluate performance. The simulations are conducted in MATLAB using the HMM toolbox, which includes functions to generate observation sequences given transition and emission probabilities, as well as an implementation of the Baum-Welch algorithm \cite{baum1970maximization}, i.e., the EM algorithm, to estimate the transition and emission probabilities using the observations. Unfortunately, even for some moderate problem sizes we considered, the streamlined MATLAB implementation of the Baum-Welch algorithm was not able to execute within reasonable amount of time, so its performance is not included here.
For the baselines, we compare with the plain NMF approach using multiplicative update \cite{Vanluyten2008} and the tensor CPD approach \cite{Sharan2017} using simultaneous diagonalization with Tensorlab \cite{tensorlab3.0}. Since we work with empirical distributions instead of exact probabilities, the result of the simultaneous diagonalization is not going to be optimal. We therefore use it to initialize the EM algorithm for fitting a nonnegative tensor factorization with KL divergence loss \cite{shashanka2008probabilistic} for refinement.

We focus on the cases when the number of hidden states $K$ is smaller than the number observed states $N$. As we explained in the introduction, even for this seemingly easier case, it is not known that we can guarantee unique recovery of the HMM parameters \emph{just from the pair-wise co-occurrence probability}. What is known is that the tensor CPD approach is able to guarantee identifiability given exact triple-occurrence probability. We will demonstrate in this section that it is much harder to obtain accurate triple-occurrence probability comparing with the co-occurrence probability. As a result, if the sufficiently scattered assumption holds for the emission probability, the estimated parameters obtained from our method are always more accurate than those obtained from tensor CPD. 

Fixing $N=100$ and $K=20$, the transition probabilities are synthetically generated from a random exponential matrix of size $K\times K$ followed by row-normalization; for the emission probabilities, approximately 50\% of the entries in the $N\times K$ random exponential matrices are set to zero before normalizing the columns, which is shown to satisfy the sufficiently scattered condition with very high probability~\cite{huang2015principled}.
We let the number of HMM realizations go from $10^6$ to $10^8$, and compare the estimation error for the transition matrix and emission matrix by the aforementioned methods. We show the total variation distance between the ground truth probabilities $\Pr[X_{t+1}|X_t]$ and $\Pr[Y_t|X_t]$ and their estimations $\widehat{\Pr}[X_{t+1}|X_t]$ and $\widehat{\Pr}[Y_t|X_t]$ using various methods. The result is shown in Figure~\ref{fig:sim_exp1}. As we can see, the proposed method indeed works best, obtaining almost perfect recovery when sample size is above $10^8$. The CPD based method does not work as well since it cannot obtain accurate estimates of the third-order statistics that it needs. 
Initialized by CPD, EM improves from CPD but the performance is still far away from the proposed method. NMF is not working well since it does not have identifiability in this case.

\section{Application: Hidden Topic Markov Model}

Analyzing text data is one of the core application domains of machine learning. There are two prevailing approaches to model text data. The classical bag-of-words model assumes that each word is \emph{independently} drawn from certain multinomial distributions. These distributions are different across documents, but can be efficiently summarized by a small number of \emph{topics}, again mathematically modeled as distributions over words; this task is widely known as \emph{topic modeling} \cite{hofmann2001unsupervised,blei2003latent}. However, it is obvious that the bag-of-words representation is oversimplified. The $n$-gram model, on the other hand, assumes that words are conditionally dependent up to a window-length of $n$. This seems to be a much more realistic model, although the choice of $n$ is totally unclear, and is often dictated by memory and computational limitations in practice---since the size of the joint distribution grows exponentially with $n$. What is more, it is somewhat difficult to extract ``topics'' from this model, despite some preliminary attempts \cite{wallach2006topic,wang2007topical}.

We propose to model a document as the realization of a HMM, in which the topics are hidden states emitting words, and the states are evolving according to a Markov chain, hence the name \emph{hidden topic Markov model} (HTMM). For a set of documents, this means we are working with a \emph{collection} of HMMs. Similar to other topic modeling works, we assume that the topic matrix is shared among all documents, meaning all the given HMMs share the same emission probability. For the bag-of-words model, each document has a specific topic distribution $\bp_d$, whereas for our model, each document has its own \emph{topic transition probability} $\bTheta_d$; as per our previous discussion, the row-sum and column-sum of $\bTheta_d$ are the same, which are also the topic probability for the specific document. The difference is the Markovian assumption on the topics rather than the over-simplifying independence assumption.

We can see some immediate advantages for the HTMM. Since the Markovian assumption is only imposed on the topics, which are not exposed to us, the observations (words) are not independent from each other, which agrees with our intuition. On the other hand, we now understand that although word dependencies exist for a wide neighborhood, we only need to work with pair-wise co-occurrence probabilities, or 2-grams. This releases us from picking a window length $n$ in the $n$-gram model, while maintaining dependencies between words well beyond a neighborhood of $n$ words. It also includes the bag-of-words assumption as a special case: If the topics of the words are indeed independent, this just means that the transition probability has the special form $\ones\bp_d^\T$. The closest work to ours is by \citet{gruber2007hidden}, which is also termed hidden topic Markov model. However, they make a simplifying assumption that the transition probability takes the form $(1-\epsilon)\bI + \epsilon\ones\bp_d^\T$, meaning the topic of the word is either the same as the previous one, or independently drawn from $\bp_d$. Both models are special cases of our general HTMM.

In order to learn the shared topic matrix $\bM$, we can use the co-occurrence statistics for the entire corpus: Denote the co-occurrence statistics for the $d$-th document as $\bOmega_d$, then $\E\bOmega_d = \bM\bTheta_d\bM^\T$; consequently
\[
\bOmega = \frac{1}{\sum_{d=1}^{D}L_d}\sum_{d=1}^{D}L_d\bOmega_d,
\]
which is an unbiased estimator for
\[
\bM\bTheta\bM^\T = \frac{1}{\sum_{d=1}^{D}L_d}\sum_{d=1}^{D}L_d\bM\bTheta_d\bM^\T,
\]
where $L_d$ is the length of the $d$-th document and $\bTheta$ is conceptually a weighted average of all the topic-transition matrices. Then we may apply Algorithm~\ref{alg:hmm_id} to learn the topic matrix. 

We illustrate the performance of our HTMM by comparing it to three popular bag-of-words topic modeling approaches: pLSA \cite{hofmann2001unsupervised}, LDA \cite{blei2003latent}, and FastAnchor \cite{arora2013practical}, which guarantees identifiability if every topic has a characteristic \emph{anchor word}. Our HTMM model guarantees identifiability if the topic matrix is \emph{sufficiently scattered}, which is a more relaxed condition than the anchor word one. On the Reuters21578 data set obtained at \cite{reuters21578}, we use the raw document to construct the word co-occurrence statistics, as well as bag-of-words representations for each document for the baseline algorithms. We use the version in which the stop-words have been removed, which makes the HTMM model more plausible since any syntactic dependencies have been removed, leaving only semantic dependencies. The vocabulary size of Reuters21578 is around $200,000$, making any method relying on triple-occurrences impossible to implement, and that is why tensor-based methods are not compared here.

\begin{figure}[t]
\centering
\includegraphics[width=0.55\linewidth]{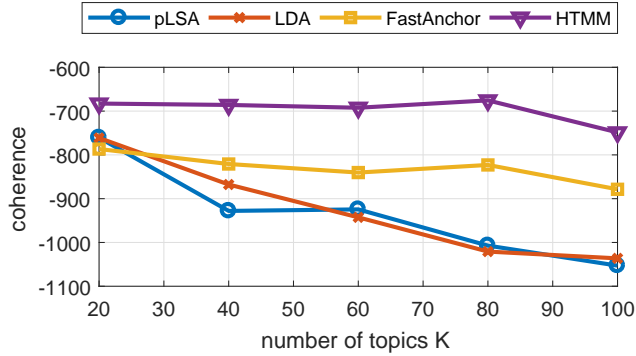}
\caption{Coherence of the topics.}
\label{fig:reuters_coh}
\end{figure}

Because of page limitations, we only show the quality of the topics learned by various methods in terms of coherence. Simply put, a higher coherence means more meaningful topics, and the concrete definition can be found in \cite{arora2013practical} and in the supplementary material. In Figure~\ref{fig:reuters_coh}, we can see that for different number of topics we tried on the entire dataset, HTMM consistently produces topics with the highest coherence. Additional evaluations can be found in the supplementary material.

\section{Conclusion}
We presented an algorithm for learning hidden Markov models in an unsupervised setting, i.e., using only a sequence of observations. Our approach is guaranteed to uniquely recover the ground-truth HMM structure using only pairwise co-occurrence probabilities of the observations, under the assumption that the emission probability is \emph{sufficiently scattered}. Unlike EM, the complexity of the proposed algorithm does not grow with the length of the observation sequence. Compared to tensor-based methods for HMM learning, our approach only requires reliable estimates of pairwise co-occurrence probabilities, which are easier to obtain. We applied our method to topic modeling, assuming each document is a realization of a HMM rather than a simpler bag-of-words model, and obtained improved topic coherence results. We refer the reader to the supplementary material for detailed proofs of the propositions and additional experimental results.

\section*{Appendix}
\appendix

\section{Proof of Proposition 1}
For categorical probabilities $\bm{p}$ and $\bm{q}$, their Kullback-Leiber divergence is defined as
\[
D_\textup{KL}(\bm{p}\|\bm{q}) = \sum_{n=1}^{N} p_n\log\frac{p_n}{q_n},
\]
and their total variation distance is defined as
\[
D_\textup{TV}(\bm{p}\|\bm{q}) = \frac{1}{2}\sum_{n=1}^{N} |p_n-q_n|.
\]

The key to prove Proposition 1 is the fact that the cooccurrence probability $\bOmega$ can be obtained by marginalizing $X_{t-1}$ in the triple-occurrence probability $\underline{\bOmega_3}$, i.e.,
\[
\bOmega(i,j) = \sum_{n=1}^{N}\underline{\bOmega_3}(n,i,j).
\]
Similarly, this holds for the cumulative estimates described in \S2 of the main paper as well,
\[
\widehat{\bOmega}(i,j) = \sum_{n=1}^{N}\widehat{\underline{\bOmega_3}}(n,i,j).
\]

Using the log sum inequality, we have that
\[
\bOmega(i,j)\log\frac{\bOmega(i,j)}{\widehat{\bOmega}(i,j)} \leq
\sum_{n=1}^{N}\underline{\bOmega_3}(n,i,j)
\log\frac{\underline{\bOmega_3}(n,i,j)}{\widehat{\underline{\bOmega_3}}(n,i,j)}.
\]
Summing both sides over $i$ and $j$, we result in
\[
D_\textup{KL}(\widehat{\bOmega}\|\bOmega) \leq 
D_\textup{KL}(\widehat{\underline{\bOmega_3}}\|\underline{\bOmega_3})
\]

Using H\"{o}lder's inequality with $L_1$-norm and $L_\infty$-norm, we have that
\[
|\bOmega(i,j)-\widehat{\bOmega}(i,j)| \leq 
\sum_{n=1}^{N}|\underline{\bOmega_3}(n,i,j) - \widehat{\underline{\bOmega_3}}(n,i,j)|.
\]
Summing both sides over $i$ and $j$ and then dividing by 2, we obtain
\[
D_\textup{TV}(\widehat{\bOmega}\|\bOmega) \leq 
D_\textup{TV}(\widehat{\underline{\bOmega_3}}\|\underline{\bOmega_3})
\]
\hfill\textbf{Q.E.D.}

\section{Proof of Proposition 2}
The volume of a hyper-ball in $\bbR^{n}$ with radius $R$ is
\[
\frac{\pi^{\frac{n}{2}}}{\Gamma(\frac{n}{2}+1)}R^{n}.
\]
The elliptical cone $\cC^*=\{ \x:\|\x\|\leq\ones^\T\x/\sqrt{K-1} \}$ intersecting with the hyperplane $\ones^\T\x=1$ is a hyperball in $\bbR^{K-1}$ with radius $\sqrt{\frac{1}{K(K-1)}}$. Therefore, the volume of the inner-ball is
\[
V_b = \frac{\pi^{\frac{K-1}{2}}}{\Gamma(\frac{K+1}{2})}(K(K-1))^{-\frac{K-1}{2}}.
\]

The nonnegative orthan intersecting with $\ones^\T\x=1$ is a regular simplex in $\bbR^{K-1}$ with side length $\sqrt{2}$. Its volume is
\[
V_s = \frac{\sqrt{K}}{(K-1)!} = \frac{\sqrt{K}}{\Gamma(K)}.
\] 

Their ratio is
\begin{align*}
\frac{V_b}{V_s} 
&= \frac{\frac{\pi^{\frac{K-1}{2}}}{\Gamma(\frac{K+1}{2})}(K(K-1))^{-\frac{K-1}{2}}}
{\frac{\sqrt{K}}{\Gamma(K)}} \\
&= \frac{1}{\sqrt{K}}\left(\frac{\pi}{K(K-1)}\right)^{\frac{K-1}{2}}
\frac{\Gamma(K)}{\Gamma(\frac{K+1}{2})}\\
&= \frac{1}{\sqrt{K}}\left(\frac{\pi}{K(K-1)}\right)^{\frac{K-1}{2}}
\frac{\Gamma(\frac{K}{2})}{2^{1-K}\sqrt{\pi}}\\
&= \frac{1}{\sqrt{\pi K}}\left(\frac{4\pi}{K(K-1)}\right)^{\frac{K-1}{2}}
\Gamma\left(\frac{K}{2}\right)
\end{align*}
\hfill\textbf{Q.E.D.}

This function of volume ratio is plotted in Figure~\ref{fig:vol_ratio}. As we can see, as $K$ increases, the volume ratio indeed goes to zero at a super-exponential rate.

\begin{figure}[t]
\centering
\includegraphics[width=.5\linewidth]{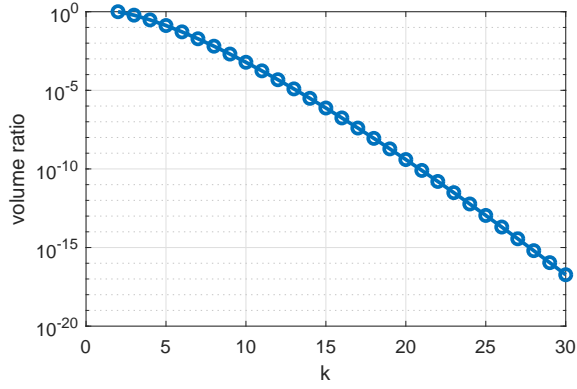}
\caption{The volume ratio between the hyperball obtained by intersecting $\mathcal{C}$ and the hyperplane $\ones^\T\x=1$ and the probability simplex, as $K$ increases.}
\label{fig:vol_ratio}
\end{figure}

\section{Derivation of \textsc{ThetaUpdate}}

It is described in \citep[\S10.2]{boyd2004convex} that for solving a convex equality constrained problem
\begin{align*}
\minimize_{x}~ & f(x) \\
\text{subject to}~ & Ax=b
\end{align*}
using the Newton's method, we start at a feasible point $x$, and the iterative update takes the form $x\leftarrow x - \alpha\Delta_\text{nt}x$, where the Newton direction is calculated from solving the KKT system
\[
\begin{bmatrix}
\nabla^2f(x) & A^\T~ \\ A & 0
\end{bmatrix}
\begin{bmatrix}
\Delta_\text{nt}x \\ d
\end{bmatrix} = 
\begin{bmatrix}
-\nabla f(x) \\ 0
\end{bmatrix}.
\]

Assuming $\nabla^2f(x)\succ0$ and $A$ has full row rank, then the KKT system can be solved via elimination, as described in \citep[Algorithm~10.3]{boyd2004convex}. Suppose $A\in\bbR^{m\times n}$, if $\nabla^2f(x)$ is diagonal, the cost of calculating $\Delta_\text{nt}x$ is dominated by forming and inverting the matrix $ADA^\T$ with $D$ being diagonal.

Now we follow the steps of \citep[Algorithm~10.3]{boyd2004convex} to derive explicit Newton iterates for solving (11). First, we re-write the part of (11) that involve $\bTheta$ here:
\begin{align*}
\minimize_{\bTheta>0}~~ & \sum_{n,\ell=1}^{N}\sum_{k,j=1}^{K}
-\varOmega_{n\ell}\varPi_{n\ell kj}^r \log\varTheta_{kj} + \lambda\sum_{k,j=1}^{K}\varXi_{kj}^r\varTheta_{kj} \\
\textrm{subject to}~~ & \ones^\T\bTheta\ones=1, \bTheta\ones=\bTheta^\T\ones.
\end{align*}

Let $\btheta=\vect(\bTheta)$, then equality constraint has the form $\bm{A}\btheta=\bm{b}$ where
\[
\bm{A} = \begin{bmatrix}
\ones^\T\otimes\ones^\T \\
\ones^\T\otimes\bI - \bI\otimes\ones^\T
\end{bmatrix}.
\]
Matrix $\bm{A}$ does not have full row rank, because the last row of $\bm{A}$ is implied by the rest. Therefore, we can discard the last equality constraint. We will keep it when calculating matrix multiplications for simpler expression, and discard the corresponding entry or column/row for other operations.

Obviously $\bm{A\theta}$ has the form
\[
\bm{A\theta} =
\begin{bmatrix}
\ones^\T\bTheta\ones \\
\bTheta\ones - \bTheta^\T\ones
\end{bmatrix},
\]
which costs $\mathcal{O}(K^2)$ flops. For a slightly more complicated multiplication
\[
\bm{A}\diag(\btheta)\bm{A}^\T = \begin{bmatrix}
\ones^\T\bTheta\ones & \ones^\T\bR^\T - \ones^\T\bTheta \\
~\bTheta\ones - \bTheta^\T\ones & \diag(\bTheta\ones + \bTheta^\T\ones) - \bTheta - \bTheta^\T~ 
\end{bmatrix},
\]
which also costs $\mathcal{O}(K^2)$ flops to compute. For $[~d_0~\bm{d}^\T~]^\T\in\bbR^{K+1}$,
\[
\bm{A}^\T[~d_0~\bm{d}^\T~]^\T = 
\vect\left(d_0\ones\ones^\T + \bm{d}\ones^\T - \ones\bm{d}^\T\right).
\]

At point $\btheta$, the negative gradient is $-\nabla f(\btheta) = \vect(\bm{G})$ where
\[ 
G_{kj} = \frac{\sum_{n,\ell=1}^{N}\varOmega_{n\ell}\varPi_{n\ell kj}^r}{\varTheta_{kj}}
- \lambda\varXi_{kj}^r,
\]
and the inverse of the Hessian $\left( \nabla^2f(\btheta) \right)^{-1} = \diag(\vect(\bm{R}))$ where
\[ 
R_{kj} = \frac{\varTheta_{kj}^2}
{\sum_{n,\ell=1}^{N}\varOmega_{n\ell}\varPi_{n\ell kj}^r}.
\]

Let
\[
\bH = \begin{bmatrix}
\ones^\T\bR\ones & \ones^\T\bR^\T - \ones^\T\bR \\
~\bR\ones - \bR^\T\ones & \diag(\bR\ones + \bR^\T\ones) - \bR - \bR^\T~ 
\end{bmatrix}
\]
and then delete the last column and row of $\bH$, and
\[
S_{kj} = R_{kj} G_{kj}
\]
\[
\bg = \begin{bmatrix}
\ones^\T\bS\ones \\ ~\bS\ones - \bS^\T\ones~
\end{bmatrix}
\]
and then delete the last entry of $\bg$. We can first solve for $\bm{d}$ by
\[
\bm{d} = \bH^{-1}\bg = [~d_0~\widetilde{\bm{d}}^{~\T}~]^\T.
\]
Then we append a zero at the end of $\bm{d}$ and define
\[
[~\bm{d}^\T~0~]^\T \rightarrow \bm{d} = [~d_0~\widetilde{\bm{d}}^{\,\T}~]^\T.
\]

The Newton direction $\Delta_\text{nt}\btheta$ can then be obtained via
\[
\Delta_\text{nt}\btheta = \left(\nabla^2f(\btheta)\right)^{-1}
\left( \bm{A}^\T\bm{d} + \nabla f(\btheta) \right).
\]

In matrix form, it is equivalent to
\[
\Delta_\text{nt}\bTheta = \bR \ast \left(d_0\ones\!\ones^\T + \widetilde{\bm{d}}\ones^\T - \ones\widetilde{\bm{d}}^{~\T} - \bm{G}\right).
\]

The in-line implementation of \textsc{ThetaUpdate} is given here.

\begin{algorithm}[h]
\caption{\textsc{ThetaUpdate}}\label{alg:theta_update}
\begin{algorithmic}[1]
\REQUIRE $\bTheta,\widetilde{\bTheta},\lambda,\rho$
\STATE $\bXi\leftarrow|\det\bTheta|\bTheta^{-\T}$
\REPEAT
\STATE $\bG\leftarrow \widetilde{\bTheta}\big/\bTheta - \lambda\bXi$
\STATE $\bR\leftarrow \bTheta\ast\bTheta\big/\widetilde{\bTheta}$
\STATE $\displaystyle \bH \leftarrow \begin{bmatrix}
\ones^\T\bR\ones & \ones^\T\bR^\T \!-\! \ones^\T\bR \\
~\bR\ones \!-\! \bR^\T\ones & \diag(\bR\ones \!+\! \bR^\T\ones) \!-\! \bR \!-\! \bR^\T~ 
\end{bmatrix}$
\STATE delete the last column and row of $\bH$
\STATE $\displaystyle \bg\leftarrow\begin{bmatrix}
\ones^\T(\bR\ast\bG)\ones \\ ~(\bR\ast\bG)\ones - (\bR\ast\bG)^\T\ones~
\end{bmatrix}$
\STATE delete the last entry of $\bg$
\STATE $\bm{d} \leftarrow \bH^{-1}\bg$
\STATE $[~d_0~\widetilde{\bm{d}}^{~\T}~]^\T\leftarrow[~\bm{d}^\T~0~]^\T$
\STATE \(
\Delta_\text{nt}\bTheta = \bR \ast \left(d_0\ones\!\ones^\T + \widetilde{\bm{d}}\ones^\T - \ones\widetilde{\bm{d}}^{~\T} - \bm{G}\right)
\)
\STATE $\bTheta\leftarrow\bTheta-\Delta_\text{nt}\bTheta$
\UNTIL{convergence}
\RETURN $\bTheta$
\end{algorithmic}
\end{algorithm}

\section{Proof of Proposition 3}
The form of Algorithm~1 falls exactly into the framework of block successive convex approximation (BSCA) algorithm proposed by \cite{razaviyayn2013unified} with only one block of variables. Invoking \citep[Theorem~4]{razaviyayn2013unified}, we have that every limit point of Algorithm~1 is a stationary point of Problem~(7). Additionally, since the constraint set of Problem~(7) is compact, \emph{any} sub-sequence has a limit point, which is also a stationary point. This proves that Algorithm~1 converges to a stationary point of Problem~(7).
\hfill
\textbf{Q.E.D.}

\section{Additional Synthetic Experiments}

In this section we conduct a similar synthetic experiment to identify HMM parameters, but with a much smaller problem size, so that we can include the classical Baum-Welch algorithm~\citep{baum1970maximization} as another baseline.
Fixing $N=16$ and $K=4$, the transition probabilities are synthetically generated from a random exponential matrix of size $K\times K$ followed by row-normalization; for the emission probabilities, the top $K\times K$ part of the $N\times K$ random exponential matrices are set to be the identity matrix before column normalization, so that it is guaranteed to be sufficiently scattered.
We let the number of HMM realizations go from $10^3$ to $10^5$, and compare the estimation error for the transition matrix and emission matrix by the aforementioned methods. We show the total variation distance between the ground truth probabilities $\Pr[X_{t+1}|X_t]$ and $\Pr[Y_t|X_t]$ and their estimations $\widehat{\Pr}[X_{t+1}|X_t]$ and $\widehat{\Pr}[Y_t|X_t]$ using various methods. The result is shown in Figure~\ref{fig:suppl_small}. 

Similar to the experiment shown in the main paper, the proposed method works the best in terms of estimating the HMM parameters, without sacrificing too much computational times. Much to one's surprise, the Baum-Welch algorithm is not working very well in terms of estimation error. This is possibly because we limit the maximum number of EM iterations to be 500 (default setting of the MATLAB implementation), which may not be enough for convergence. What is expected is that the computational time of Baum-Welch grows linearly with respect to the length of the HMM observations, while other methods are independent from it.

\begin{figure}[t]
\centering
\hspace*{-.1\linewidth}
\includegraphics[width=1.3\linewidth]{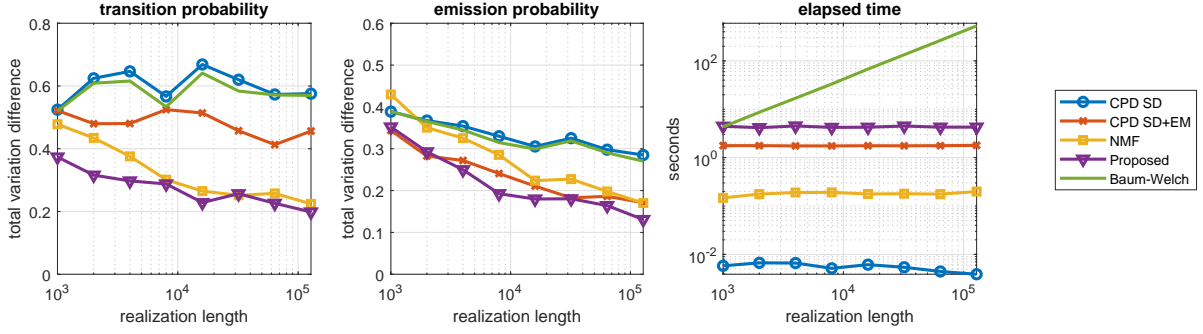}
\caption{The total variation difference between the ground truth and estimated transition probability (left) and emission probability (middle), and the elapsed time (right) for $N=16$ and $K=4$. The total variation difference of the emission probabilities is calculated as $\frac{1}{2K}\|\bM_\natural-\bM_\star\|_1$, since each column of the matrices indicates a (conditional) probability, and the total variation difference is equal to one half of the $L_1$-norm; and similarly for that of the transition probabilities after rescaling the rows of $\bOmega_\natural$ and $\bOmega_\star$ to sum up to one. The result is averaged over 10 random problem instances.}
\label{fig:suppl_small}
\end{figure}

\begin{figure}[t]
\centering
\includegraphics[width=.5\linewidth]{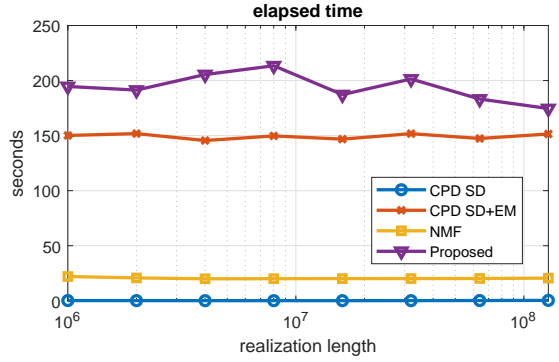}
\caption{The elapsed time for the synthetic experiment with $N=100$ and $K=20$ as in the main paper.}
\label{fig:suppl_time}
\end{figure}

An interesting remark is that when $T=12,800$, the per-iteration elapsed time of Baum-Welch is approximately 1 second. Recall that each iteration of Baum-Welch calls for the forward-backward algorithm, with complexity $\mathcal{O}(K^2T)$. This means for the problem size considered in the main paper, each iteration of Baum-Welch takes approximately 4 minutes to 7 hours, depending on the realization length. This is clearly not feasible in practice.

We also present the elapsed time of the four algorithms excluding the Baum-Welch algorithm for the case considered in the main paper, i.e., $N=100$ and $K=20$. Similar to the timing result shown in Figure~\ref{fig:suppl_small}, the proposed method takes the longest time compared to the other three, but not significantly; also recall that the propose method works considerably better in terms of estimation accuracy. 

\section{Additional HTMM Evaluations}

In the main body of the paper we showed that HTMM is able to learn topics with higher quality using pairwise word cooccurrences. The quality of topics is evaluated using coherence, which is defined as follows. For each topic, a set of words $\mathcal{V}$ is picked (here we pick the top 20 words with the highest probability of appearing). We calculate the number of documents each word $v_1$ appears $\text{freq}(v_1)$ and the number of documents two words $v_1$ and $v_2$ both appear $\text{freq}(v_1,v_2)$. The coherence of that topic is calculated as
\[
\sum_{v_1,v_2\in\mathcal{V}} \log\left(\frac{\text{freq}(v_1,v_2)+\epsilon}{\text{freq}(v_1)}\right) .
\]
The intuition is that if both $v_1$ and $v_2$ both have high probability of appearing in a topic, then they have high probability of co-occurring in a document as well; hence a higher value of coherence indicates a more indicative topic. The coherence of the individual topics are then averaged to get the coherence for the entire topic matrix.

\begin{figure}[t]
\centering
\includegraphics[width=.5\linewidth]{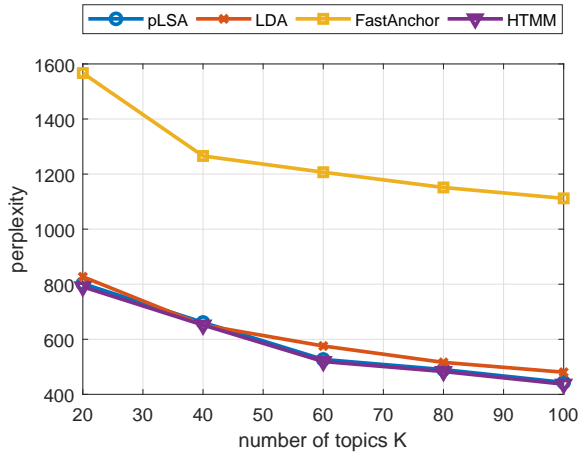}
\caption{The perplexity of different models as number of topics $K$ increases.}
\label{fig:perp}
\end{figure}

Here we show some more evaluation results.
Using the learned topic matrix, we can see how it fits the data directly from \emph{perplexity}, defined as \cite{blei2003latent}
\[
\exp\left( -\frac{\sum_d\log p(\text{doc}_d)}{\sum_d L_d} \right).
\]
A smaller perplexity means the probability model fits the data better. As seen in Figure~\ref{fig:perp}, HTMM gives the smallest perplexity. Notice that since HTMM takes word ordering into account, it is not fair for the other methods to take the bag-of-words representation of the documents. The bag-of-words model is essentially multinomial, whose pdf includes a scaling factor $\frac{n!}{n_1!...n_K!}$ for different combinations of observation orderings. In our case this factor is not included since we \emph{do} know the word ordering in each document. For HTMM the log-likelihood is calculated efficiently using the forward algorithm.

This result is not surprising. Even using the same topic matrix, a bag-of-words model tries to find a $K$-dimensional representation for each document, whereas HTMM looks for a $K^2$-dimension representation. One may wonder if it is causing over-fitting, but we argue that it is not. First of all, we have see that in terms of coherence, HTMM learns a topic matrix with higher quality. For learning feature representations for each document, we showcase the following result. Once we have the topic-word probabilities and topic weights or topic transition probability, we can infer the underlying topic for each word. For bag-of-words models, each word only has one most probable topic in a document, no matter where it appears. For HTMM, once we learn the transition and emission probability, the topic of each word can be optimally estimated using the Viterbi algorithm. For one specific news article from the Reuters21578 data set, the topic inference given by pLSA is:

\begin{framed}
\teight{china} \tfive{daily} \tfour{vermin eat} \tseven{pct} \tone{grain stocks survey provinces} \tseven{and} \tone{cities showed} \tfour{vermin consume} \tseven{and pct} \teight{china} \tone{grain stocks} \teight{china} \tfive{daily} \tone{that} \tseven{each} \tone{year} \tseven{mln} \teight{tonnes} \tseven{pct} \teight{china} \tfive{fruit} \ttwo{output} \tone{left} \tseven{rot and mln} \teight{tonnes} \tseven{pct} \tfour{vegetables} \tfive{paper} \teight{blamed} \tseven{waste inadequate} \tfive{storage} \tseven{and} \tone{bad} \tfour{preservation} \tseven{methods} \tone{government} \tseven{had launched} \tfive{national} \tone{programme reduce} \tseven{waste calling for} \tone{improved} \tfive{technology storage} \tseven{and} \tfour{preservation} \tseven{and} \tone{greater production} \tfour{additives} \tfive{paper} \tseven{gave details}
\end{framed}

The word topic inference given by HTMM is:

\begin{framed}
\tfour{china} \ttwo{daily} \tfour{vermin eat} \teight{pct} \tfour{grain} \ttwo{stocks} \tfour{survey provinces and cities showed vermin consume and pct} \teight{china} \tfour{grain stocks} \ttwo{china} \tfour{daily} \ttwo{that} \tfour{each} \tseven{year} \tfour{mln} \teight{tonnes} \tfour{pct} \teight{china} \tfour{fruit output} \ttwo{left} \tfour{rot and mln} \teight{tonnes} \tfour{pct} \teight{vegetables} \tfour{paper blamed waste} \tseven{inadequate} \tfour{storage} \ttwo{and} \tfour{bad preservation methods government had launched national programme} \teight{reduce} \tfour{waste} \tseven{calling} \tfour{for improved} \teight{technology} \tfour{storage} \ttwo{and} \tfour{preservation and greater production} \ttwo{additives} \tfour{paper gave details}
\end{framed}

As we can see, HTMM gets much more consistent and smooth inferred topics, which agrees with human understandings.

\section{Learning HMMs from Triple-occurrences}

Finally, we show a stronger identifiability result for learning HMMs using triple-occurrence probabilities.

\begin{theorem}
Consider a HMM with $K$ hidden states and $N$ observable states. Suppose the emission probability $\Pr[Y_t|X_t]$ is generic (meaning probabilities not satisfying this condition form a set with Lebesgue measure zero), the transition probabilities $\Pr[X_{t+1}|X_t]$ are linearly independent from each other, and each conditional probability $\Pr[X_{t+1}|X_t]$ contains no more than $N/2$ nonzeros. Then this HMM can be uniquely identified from its triple-occurrence probability $\Pr[Y_{t-1},Y_t,Y_{t+1}]$, up to permutation of the hidden states, for $K\leq \frac{N^2}{16}$.
\end{theorem}

\begin{proof}
It is clear that identifiability holds when $K\leq N$, so we focus on the case that $N<K\leq \frac{N^2}{16}$.

As we explained in \S1.1, the triple-occurrence probability can be factored into
\[
\Pr[Y_{t-1},Y_t,Y_{t+1}] = \sum_{k=1}^{K}\Pr[X_t=x_k]\Pr[Y_{t-1}|X_t=x_k] 
 \Pr[Y_{t}|X_t=x_k]\Pr[Y_{t+1}|X_t=x_k].
\]
Using tensor notations, this is equivalent to
\[
\underline{\bOmega_3} = \cpd{\bm{p};\bm{L},\bM,\bm{N}},
\]
where
\begin{align*}
p_k & = \Pr[X_t=x_k], \\
L_{nk} &= \Pr[Y_{t-1}=y_n|X_t=x_k], \\
N_{nk} &= \Pr[Y_{t+1}=y_n|X_t=x_k].
\end{align*}
Let $\overline{\bTheta}$ denote the row scaled version of $\bTheta$ so that each row sums to one, then $\overline{\bTheta}$ denotes the transition probability. Then we have
\begin{equation}\label{eq:L=MT}
\bm{L}=\bM\overline{\bTheta}^\T.
\end{equation}

Since $\bM$ is generic and $\overline{\bTheta}$ is full rank, both $\bm{L}$ and $\bm{N}$ are generic as well. The latest tensor identifiability result by \citet[Theorem 1.1]{chiantini2012generic} shows that for a $N\times N\times N$ tensor with generic factors, the CPD $\underline{\bOmega_3} = \cpd{\bm{p};\bm{L},\bM,\bm{N}}$ is essentially unique if
\[
K\leq 2^{ 2\lfloor\log_2N\rfloor - 2 },
\]
or with a slightly worse bound
\[
K\leq \frac{N^2}{16}.
\]

This does not mean that any non-singular $\overline{\bTheta}$ can be uniquely recovered in this case. Equation~\eqref{eq:L=MT} is under-determined. A natural assumption to achieve identifiability is that each row of $\overline{\bTheta}$, i.e., each conditional transition probability $\Pr[X_{t+1}|X_t]$, can take at most $N/2$ non-zeros. In the context of HMM, this means that at a particular hidden state, there are only a few possible states for the next step, which is very reasonable. For a generic $\bM$,
\[
\spark(\bM) = \krank(\bM) + 1 = N+1.
\]
\citet{Donoho2003} showed that for such a $\bM$, and a vector $\bm{\theta}$ with at most $N/2$ nonzeros, $\bm{\theta}$ is the unique solution with at most $N/2$ nonzeros that satisfies $\bM\bm{\theta} = \bm{\ell}$. Therefore, if we seek for the sparsest solution to the linear equation \eqref{eq:L=MT}, we can uniquely recover $\bTheta$ as well.

\end{proof}

\bibliographystyle{plainnat}
\bibliography{hmm_refs}

\end{document}